\documentclass[conference]{IEEEtran}
\IEEEoverridecommandlockouts
\usepackage{cite}
\usepackage{amsmath,amssymb,amsfonts}
\usepackage{algorithmic}
\usepackage{graphicx}
\usepackage{textcomp}
\usepackage{xcolor}
\usepackage{booktabs}
\usepackage{amsthm}

\newtheorem{theorem}{Theorem}
\newtheorem{lemma}[theorem]{Lemma}
\newtheorem{proposition}[theorem]{Proposition}
\newtheorem{corollary}[theorem]{Corollary}
\newtheorem{definition}[theorem]{Definition}
\newtheorem{assumption}[theorem]{Assumption}
\newtheorem{remark}[theorem]{Remark}

\newcommand{\R}{\mathbb{R}}
\newcommand{\E}{\mathbb{E}}

\newcommand{\given}{\,|\,}
\newcommand{\xlag}{x^{(K)}}
\newcommand{\GNAVAR}{\textsc{G-NAVAR}}
\newcommand{\NAVAR}{\textsc{NAVAR}}
\newcommand{\supp}{\operatorname{supp}}
\newcommand{\reff}{r_{\mathrm{eff}}}

\begin{document}

\title{When Are Neural Interaction Discoveries Real?\\ Identifiability, Recoverability, and a Pre-Fit Diagnostic}

\author{\IEEEauthorblockN{Valentina V. Kuskova}
\IEEEauthorblockA{\textit{Lucy Family Institute for Data \& Society} \\
\textit{University of Notre Dame}\\
Notre Dame, IN, USA \\
vkuskova@nd.edu}
\and
\IEEEauthorblockN{Dmitry Zaytsev}
\IEEEauthorblockA{\textit{Lucy Family Institute for Data \& Society} \\
\textit{University of Notre Dame}\\
Notre Dame, IN, USA \\
zaytsevdi2@gmail.com}
\and
\IEEEauthorblockN{Michael Coppedge}
\IEEEauthorblockA{\textit{Department of Political Science} \\
\textit{University of Notre Dame}\\
Notre Dame, IN, USA \\
mcoppedg@nd.edu}
}
\maketitle

\begin{abstract}
When a neural time-series model reports that one variable modulates another's effect on a target, is the discovered interaction a property of the data or an artifact of model flexibility? We argue that this is fundamentally a question of identifiability, governed by the geometry of the observed input support rather than by the specific neural architecture. We study the problem in a multiplicative-gating extension of neural additive vector autoregression (\GNAVAR), in which source contributions are modulated by other lagged variables. We show that representational capacity is not identifiability: dependent inputs induce leakage between edge-specific interaction terms, and low-dimensional support permits distinct interaction decompositions that agree on the observed data while differing elsewhere.
We then prove a population identifiability theorem for normalized minimal \GNAVAR{} decompositions under explicit support conditions, including settings with shared modulators. The theory yields a simple practitioner-facing diagnostic: the effective rank of the joint lag-block covariance predicts, before fitting, whether interaction recovery is feasible for a given candidate set. When the candidate set is unknown, a two-seed stability check provides a practical operational test.
The same support condition organizes empirical outcomes into the three states predicted by the theory. In Beijing air-quality data, rich support yields stable recovery of interaction structure. In development-indicator data, support is rich but recovery is unstable across seeds, indicating no stably recoverable interaction structure despite feasibility in principle. In realized-volatility data, collapsed support prevents recovery altogether. Together, these results show that interaction recoverability depends on support geometry, that effective rank provides a practical pre-fit diagnostic, and that instability across independent fits is a characteristic signature of non-identifiable interaction discovery. The identifiability phenomenon, the support condition, and the instability signature are model-agnostic; \GNAVAR{} is the vehicle that makes them provable.
\end{abstract}


\section{Introduction}

Neural models of multivariate time series increasingly report not only which variables influence a target, but also how those influences interact: that the effect of one variable is modulated by the level of another. Such interaction discoveries are attractive because they appear mechanistic: one can read a causal-looking narrative directly from the fitted model. Yet a flexible model may admit multiple interaction decompositions that fit the observed data equally well. In that case, the recovered modulator reflects the optimization path rather than the underlying data-generating process. The central question is therefore not whether a model can represent interactions, but whether the interaction structure is uniquely recoverable from observational data. We argue that this is fundamentally a problem of \emph{identifiability}, and that the answer is governed largely by the geometry of the observed input support rather than by the particular neural architecture used to model it.

To study this question in a setting that is both expressive and analyzable, we consider a multiplicative-gating extension of neural additive vector autoregression (\NAVAR)~\cite{bussmann2021navar}. Neural additive autoregressive models and related neural Granger-causality methods~\cite{tank2022neural} decompose a target's dynamics into interpretable per-source contributions but cannot naturally express \emph{effect modification}, where the influence of one variable depends on the level of another. We therefore augment each source contribution with multiplicative gates driven by candidate modulators, yielding gated neural additive vector autoregression (\GNAVAR). The resulting edge-and-gate representation appears interpretable: a nontrivial gate suggests that one variable modulates the effect of another. The question addressed in this paper is whether such a reading is justified.

Representational capacity is not identifiability. A model class may be rich enough to fit the data while still admitting multiple observationally equivalent interaction decompositions. We show that this ambiguity arises through two distinct mechanisms. First, under dependent inputs, interaction structure can leak between edge-specific terms through shared lower-order variation. Second, when the observed support is effectively low-dimensional, distinct interaction decompositions may agree everywhere on the observed data while differing off-support. In either case, interaction recovery becomes a property of the model parameterization rather than of the observational distribution.

Our contributions are as follows.

\begin{itemize}
\item We characterize the gauge symmetries of multiplicative interaction decompositions, including scale redistribution, modulator permutation, and insertion of trivial gates. Gate normalization and a minimality condition yield a well-defined notion of identifiability.

\item We establish two impossibility results showing that expressive interaction models need not yield identifiable interaction recovery. Dependent inputs permit leakage between edge-specific interaction terms, while low-dimensional support admits observationally equivalent decompositions.

\item We prove a population identifiability theorem for normalized minimal \GNAVAR{} decompositions under explicit support conditions. A fully rigorous disjoint-support result is complemented by a more general construction that permits shared modulators through a hierarchically orthogonal functional decomposition.

\item We derive a practical pre-fit diagnostic based on the effective rank of the joint lag-block covariance and complement it with a two-seed stability check. Together they distinguish three empirical states predicted by the theory: recoverable interaction structure under rich support, unstable recovery despite rich support, and unrecoverable interaction discovery under support collapse.

\item We validate the theory on synthetic data and on three real domains representing these three states: Beijing air quality, development indicators, and realized volatility.
\end{itemize}

Although the formal results are developed for \GNAVAR, the broader contribution is not architectural. We argue that interaction discovery in neural autoregression is fundamentally a representation-identifiability problem whose feasibility is determined by input-support geometry. This perspective connects neural Granger causality to functional ANOVA~\cite{hooker2007generalized,chastaing2012generalized}, tensor-factorization identifiability~\cite{kruskal1977threeway,anandkumar2014tensor}, nonlinear ICA~\cite{hyvarinen2019nonlinear,khemakhem2020variational}, and impossibility results for unsupervised disentanglement~\cite{locatello2019challenging}.

\section{The G-NAVAR Model}
\label{sec:model}

We study interaction recovery in a multiplicative-gating extension of additive neural autoregression. Each source contributes through a base function, while candidate modulators act through multiplicative gates that scale that contribution. Let $x_t \in \R^{n}$ be a stationary multivariate time series with lag order $K$, and write $\xlag_{j,t} = (x_{j,t-1},\ldots,x_{j,t-K}) \in \R^{K}$ for the lag block of coordinate $j$. For a target coordinate $i$, the \GNAVAR{} model is
\[
 \hat x_{i,t} = \beta_i + \sum_{j} f_{ij}(\xlag_{j,t}) \prod_{k \in M_{ij}} g_{ijk}(\xlag_{k,t}),
\]
where $f_{ij}:\R^{K}\to\R$ is the base function of the edge from source $j$, the set $M_{ij}$ lists the modulators of that edge, and each gate $g_{ijk}:\R^{K}\to\R_{>0}$ is strictly positive (parameterized as $g=\exp(\eta)$ for an unconstrained network $\eta$). 

The quantity $F_{ij} = f_{ij}\prod_{k\in M_{ij}} g_{ijk}$ is the effective contribution of source $j$ and will be referred to as the edge product. Setting all $M_{ij}=\varnothing$ recovers additive \NAVAR.

\begin{definition}[\GNAVAR{} representation]
\label{def:repr}
A representation of $h_i(\xlag)=\E[x_{i,t}\given \xlag_{t}]$ is a choice of intercept $\beta_i$, base functions $\{f_{ij}\}$, modulator sets $\{M_{ij}\}$, and gates $\{g_{ijk}\}$ for which $h_i = \beta_i + \sum_j F_{ij}$.
\end{definition}

\begin{definition}[Observational equivalence]
\label{def:obseq}
Two representations are observationally equivalent if they induce the same conditional mean $h_i$ almost surely under the observational distribution $p$.
\end{definition}

The model above admits multiple representations of the same conditional mean. Some correspond to benign symmetries of the parameterization, while others represent genuine failures of identifiability. We first remove the former through normalization and minimality before asking whether the remaining interaction structure is uniquely recoverable from the observational distribution.

\section{Gauge Symmetries and Normalization}
\label{sec:gauge}

Before identifiability can be studied, we must account for transformations that alter the decomposition while leaving the induced conditional mean unchanged.

\begin{lemma}[Scale gauge]
\label{lem:scale}
For any active edge and constants $c, c_k>0$ with $c\prod_k c_k = 1$, replacing $f_{ij}$ by $c\,f_{ij}$ and each $g_{ijk}$ by $c_k g_{ijk}$ leaves $F_{ij}$ unchanged.
\end{lemma}

\begin{lemma}[Permutation and trivial-gate gauges]
\label{lem:perm}
Permuting the modulator index $k$ within an edge leaves $F_{ij}$ unchanged, as does inserting or deleting a gate identically equal to one.
\end{lemma}

\begin{assumption}[Normalization]
\label{ass:normalization}
For each nontrivial gate, $\E[g_{ijk}(x_k)]=1$. The intercept is $\beta_i=\E[h_i(\xlag)]$, and edge terms are centered under the decomposition convention used below. No additional normalization of the base functions $f_{ij}$ is required. Proposition~\ref{prop:edge-product} shows that gate normalization alone removes the multiplicative scale ambiguity.
\end{assumption}

\begin{definition}[Essential modulator and minimal representation]
\label{def:minimal}
A modulator $k\in M_{ij}$ is \emph{essential} for edge $(i,j)$ if the edge product $F_{ij}$ is nonconstant in $x_k$ on an open product subset of the support, holding the other arguments fixed. A representation is \emph{minimal} if every listed modulator is essential. Minimality excludes modulator assignments that do not affect the edge product and therefore removes the trivial-gate ambiguity identified above.
\end{definition}

\section{Why Naive Edge Isolation Fails}
\label{sec:fail}

Before giving sufficient conditions for identifiability, we show that observational equivalence alone is insufficient to recover interaction structure. Two distinct obstructions arise: dependence-induced leakage between source effects and geometric degeneracy of the observed support.

\begin{proposition}[Conditioning does not isolate source effects under dependence]
\label{prop:leakthrough}
Suppose $x_1$ and $x_2$ are dependent. Then conditioning on $x_1$ does not uniquely identify the structural contribution associated with $x_1$ from an additive signal $A(x_1)+B(x_2)$. The conditional mean absorbs the projection of $B(x_2)$ onto functions of $x_1$, making that component observationally indistinguishable from a direct $x_1$ effect.
\end{proposition}

\begin{proof}
Let $B(x_2)$ be square-integrable and nonconstant, and set $m(x_1)=\E[B(x_2)\given x_1]$. Then $A(x_1)+B(x_2) = \{A(x_1)+m(x_1)\}+\{B(x_2)-m(x_1)\}$. The first bracket is a function of $x_1$; the second has conditional mean zero given $x_1$. The part $m(x_1)$ of the $x_2$-effect that correlates with $x_1$ is, from the standpoint of conditioning on $x_1$, identical to a direct $x_1$-effect. The closed subspaces of functions of $x_1$ and of $x_2$ need not be orthogonal. Consequently, their sum does not admit a unique decomposition into structural source components. Unique attribution therefore requires an additional convention, such as a direct-sum or hierarchically orthogonal decomposition.
\end{proof}

The second obstruction is geometric. When the joint support of a modulator set is effectively low-dimensional, distinct gates may agree everywhere on the observed support while differing elsewhere. We state the result formally as Corollary~\ref{cor:lowdim} after the main theorem, where it appears naturally as a converse to the support conditions required for identifiability.

\section{Identifiability Theory}
\label{sec:theory}

\subsection{Edge separability}

The identifiability result rests on two ingredients. First, contributions associated with different source edges must remain distinguishable as functional objects; otherwise interaction structure can leak between edges. Second, the observed support must be rich enough to distinguish competing gate factorizations. We develop these requirements in turn.

We begin with the problem of separating contributions associated with different source edges. Let $\mathcal H_S\subset L^2(p)$ denote the closed subspace of square-integrable functions depending only on variables indexed by $S$. For each source $j$ let $M_{ij}^{\max}$ be a fixed candidate modulator set, specified in advance and independent of any representation, and write $S_j^{\max}=\{j\}\cup M_{ij}^{\max}$.

\begin{assumption}[Edge direct-sum separability]
\label{ass:directsum}
For fixed target $i$, the ambient edge spaces $\mathcal V_{ij} = \mathcal H_{S_j^{\max}}\ominus\{\text{constants}\}$ are fixed in advance, and every admissible representation of $F_{ij}$ (with $M_{ij}\subseteq M_{ij}^{\max}$) lies in this common $\mathcal V_{ij}$. The ambient edge spaces form a direct sum after centering: if $\sum_j V_j=0$ $p$-a.s.\ with $V_j\in\mathcal V_{ij}$ and $\E[V_j]=0$, then $V_j=0$ $p$-a.s.\ for every $j$.
\end{assumption}

Fixing the ambient edge spaces in advance ensures that observationally equivalent representations are compared within the same functional space. Consequently, if two representations assign different modulator sets to the same edge, their difference still lies in the common space $\mathcal V_{ij}$.

\begin{proposition}[Direct sum under disjoint edge supports and source independence]
\label{prop:directsum-disjoint}
Suppose the source lag-blocks $(x_j)_{j=1}^N$ are mutually independent as random vectors, and the edge variable sets $S_j=\{j\}\cup M_{ij}$ are pairwise disjoint. Then the centered edge spaces $\mathcal V_{ij}$ are pairwise orthogonal in $L^2(p)$, and Assumption~\ref{ass:directsum} holds.
\end{proposition}

\begin{proof}
By disjointness and mutual independence, the vectors $(x_{S_j})_j$ are mutually independent. For $j\neq k$ and centered $V_j,V_k$, $\langle V_j,V_k\rangle = \E[V_jV_k]=\E[V_j]\E[V_k]=0$, so the spaces are pairwise orthogonal. If $\sum_j V_j=0$, taking the inner product with $V_k$ gives $\|V_k\|^2=0$, so $V_k=0$ for every $k$.
\end{proof}

\begin{remark}[Scope and the shared-modulator obstruction]
\label{rem:disjoint-scope}
Proposition~\ref{prop:directsum-disjoint} forbids any variable from appearing in two edges. It is also necessary in the following sense: under independence but with a shared modulator $x_m$, the direct sum can fail. For any centered nonconstant $c(x_m)$, setting $V_j=c(x_m)$ and $V_{j'}=-c(x_m)$ gives $V_j+V_{j'}=0$ with both nonzero. The obstruction is overlap between edge spaces, not dependence; independence does not remove it. The next subsection removes this restriction by constructing smaller source-anchored spaces that permit shared modulators while preserving uniqueness.
\end{remark}

Edge separability alone is not sufficient. Even when edge contributions are uniquely assigned, distinct gate factorizations may remain observationally equivalent if the observed support is too limited.

\subsection{Support and regularity}

\begin{assumption}[Full-dimensional connected product support]
\label{ass:support}
For every $S$ that appears as $\{j\}\cup M_{ij}$ or a subset thereof, the marginal law of $x_S$ is absolutely continuous, with support an open connected product set $U_S=\prod_{k\in S}U_k$ ($U_k\subset\R^{K_k}$ open connected), and density strictly positive on $U_S$.
\end{assumption}

\begin{assumption}[Regularity and nondegeneracy]
\label{ass:regularity}
Base functions and gates are continuous; gates are strictly positive; and for an active edge $f_{ij}$ vanishes on no open subset of $\supp(x_j)$.
\end{assumption}

\begin{lemma}[Continuity upgrade under full-dimensional support]
\label{lem:continuity-upgrade}
Let $U\subset\R^d$ be open and connected, and $p$ absolutely continuous with density strictly positive on $U$ (as guaranteed by Assumption~\ref{ass:support}). If $\phi,\psi:U\to\R$ are continuous and $\phi=\psi$ $p$-a.s., then $\phi=\psi$ everywhere on $U$.
\end{lemma}

\begin{proof}
$A=\{\phi\neq\psi\}$ is open; if nonempty it has positive Lebesgue measure, hence positive $p$-measure, contradicting $\phi=\psi$ a.s.
\end{proof}

\subsection{Shared modulators via source-anchored functional ANOVA}
\label{sec:shared-mod}

The direct-sum construction above is fully rigorous but restrictive because it forbids shared modulators. In practice, however, a variable may plausibly modulate multiple edges. The difficulty is that shared modulators create overlap between edge spaces, destroying the uniqueness argument used in Proposition~\ref{prop:directsum-disjoint}. To recover uniqueness, we construct smaller source-anchored spaces using a hierarchically orthogonal functional decomposition (HOFD).

\begin{assumption}[HOFD regularity]
\label{ass:hofd}
The joint law $p$ satisfies the boundedness and non-degeneracy conditions of \cite{chastaing2012generalized} ensuring the projections defining the hierarchically orthogonal decomposition are well defined and the decomposition is unique.
\end{assumption}

\begin{proposition}[Hierarchically orthogonal decomposition~\cite{hooker2007generalized,chastaing2012generalized}]
\label{prop:hoa}
Under Assumption~\ref{ass:hofd}, every $h\in L^2(p)$ admits a unique decomposition $h=\sum_{T} h_T$ with $h_T\in\mathcal H_T$ orthogonal to every function of a strict subset of $T$. Equivalently, the only such decomposition of the zero function is the all-zero one. Write $\mathcal H_T^0$ for the admissible top-order component space at $T$.
\end{proposition}

The components for non-nested $T$ need not be mutually orthogonal; only uniqueness is used. Assumption~\ref{ass:hofd} is a genuine extra hypothesis: the connected-product-support condition of Assumption~\ref{ass:support} does not by itself imply it.

\begin{definition}[Source-anchored edge space]
\label{def:source-anchored}
The source-anchored edge space is $\mathcal V_{ij}^{\mathrm{sa}}=\operatorname{span}\{\mathcal H_T^0 : j\in T\subseteq S_j\}$, the span of HOFD components on $S_j$ that contain the source $j$. The source-anchored part of $F_{ij}$ is $F_{ij}^{\mathrm{sa}}=\sum_{j\in T\subseteq S_j}(F_{ij})_T$, a selection of HOFD coordinates (not an orthogonal projection). Intuitively, each edge retains only those interaction components that contain its own source variable.
\end{definition}

\begin{assumption}[No source is a modulator of another edge]
\label{ass:no-source-mod}
For all $j\neq j'$, $j\notin M_{ij'}$ (equivalently $j\notin S_{j'}$). Modulators may still be shared: $M_{ij}\cap M_{ij'}$ may be nonempty. 
\end{assumption}
This condition prevents a source variable from serving simultaneously as another edge's modulator, which would reintroduce overlap between source-anchored spaces.

\begin{proposition}[Direct sum of source-anchored spaces under shared modulators]
\label{prop:directsum-shared}
Under Assumptions~\ref{ass:no-source-mod} and \ref{ass:hofd}, the source-anchored spaces $\{\mathcal V_{ij}^{\mathrm{sa}}\}_j$ form a direct sum.
\end{proposition}

\begin{proof}
Let $\mathcal T_j=\{T:j\in T\subseteq S_j\}$. If $T\in\mathcal T_j\cap\mathcal T_{j'}$ with $j\neq j'$, then $j\in T\subseteq S_{j'}$, so $j\in S_{j'}$, contradicting Assumption~\ref{ass:no-source-mod}; hence the $\mathcal T_j$ are pairwise disjoint. Each $V_j=\sum_{T\in\mathcal T_j}(V_j)_T$ is a sum of HOFD components indexed by $\mathcal T_j$; if $\sum_j V_j=0$, the disjointness makes this a single HOFD decomposition of zero in which each index appears once, so by uniqueness (Proposition~\ref{prop:hoa}) every component vanishes and each $V_j=0$.
\end{proof}

\begin{definition}[Structural modulation]
\label{def:structural-mod}
A variable $k$ \emph{modulates} edge $(i,j)$ if $F_{ij}^{\mathrm{sa}}$ has a nonzero component on some $T\supseteq\{j,k\}$; that is, $x_k$ enters $F_{ij}$ through a genuine interaction with the source $x_j$. In other words, modulation is defined through source–modulator interaction structure rather than through the presence of a gate parameter.
\end{definition}

A gate modulates the source's effect precisely through the source--modulator interaction; a pure $x_k$ contribution (possibly shared across edges) is a separate additive term, not modulation of edge $(i,j)$. The structural modulation set need not coincide with the listed gate set $M_{ij}$ of Definition~\ref{def:minimal}; bridging them needs an explicit faithfulness condition. Recovering listed gates therefore requires a bridge between functional interaction structure and the parameterization itself.

\begin{assumption}[Gate--interaction faithfulness]
\label{ass:faithful}
For every edge and every $k$, $k\in M_{ij}$ iff $F_{ij}^{\mathrm{sa}}$ has a nonzero component on some $T\supseteq\{j,k\}$.
\end{assumption}

Assumption~\ref{ass:faithful} holds cleanly under input independence: the source-anchored $\{j,k\}$ interaction is nonzero exactly when $g_{ijk}$ is nonconstant. Under dependence, the cancellations required for a nonconstant gate to produce no source--modulator interaction are non-generic. We therefore state faithfulness as a hypothesis rather than derive it.

\begin{proposition}[Structural-modulation-set recovery with shared modulators]
\label{prop:modset-shared}
Under Assumptions~\ref{ass:no-source-mod}, \ref{ass:hofd}, \ref{ass:support}, \ref{ass:regularity}, if two representations of $h_i$ are observationally equivalent, then their source-anchored parts agree for every edge, $F_{ij}^{\mathrm{sa}}=F_{ij}'^{\,\mathrm{sa}}$, so the structural modulation sets coincide. If gate--interaction faithfulness (Assumption~\ref{ass:faithful}) holds for both, the listed modulator sets $M_{ij}$ coincide up to permutation and trivial-gate equivalence.
\end{proposition}

\begin{proof}
Let the two representations have variable sets $S_j,S'_j$ and define $\mathcal T_j^{\ast}=\{T:j\in T\subseteq S_j\}\cup\{T:j\in T\subseteq S'_j\}$. If $T\in\mathcal T_j^{\ast}\cap\mathcal T_{j'}^{\ast}$ with $j\neq j'$, then $j\in S_{j'}$ or $j\in S'_{j'}$, contradicting Assumption~\ref{ass:no-source-mod} (imposed on both); so the $\mathcal T_j^{\ast}$ are pairwise disjoint. Observational equivalence gives $\sum_j(F_{ij}-F'_{ij})=0$. Decomposing into HOFD components (Proposition~\ref{prop:hoa}), each source-anchored index $T\in\mathcal T_j^{\ast}$ appears for exactly one edge, so uniqueness forces its component to vanish; summing gives $F_{ij}^{\mathrm{sa}}=F_{ij}'^{\,\mathrm{sa}}$. The structural modulation set is a function of the source-anchored part alone, so the sets coincide; under Assumption~\ref{ass:faithful} each listed set equals the common structural set.
\end{proof}

\begin{remark}[Scope and Limitations]
\label{rem:shared-scope}
This is a conditional strengthening of Proposition~\ref{prop:directsum-disjoint}. (i) It rests on HOFD uniqueness (Proposition~\ref{prop:hoa}, via Assumption~\ref{ass:hofd}), a genuine extra hypothesis not implied by Assumption~\ref{ass:support}. (ii) Assumption~\ref{ass:no-source-mod} permits shared modulators but leaves open the case where a source is also another edge's modulator. (iii) Structural-modulation-set recovery is unconditional within (i)--(ii); upgrading to listed-set recovery needs gate--interaction faithfulness (Assumption~\ref{ass:faithful}), clean under independence and generic (in the measure-zero sense above) under dependence.
\end{remark}

\subsection{Single-edge identifiability and the main theorem}
The previous subsection establishes when edge-level contributions can be uniquely assigned. We now address the remaining question: once an edge product is identified, can its multiplicative factors be recovered uniquely? The following lemma provides the key factorization result.

\begin{lemma}[Separated multiplicative factorization]
\label{lem:factorization}
Let $U=\prod_{r=1}^m U_r$ be a product of open connected sets, and $a_r,b_r:U_r\to\R_{>0}$ continuous with $\prod_r a_r(z_r)=\prod_r b_r(z_r)$ on $U$. Then $a_r=c_r b_r$ with $\prod_r c_r=1$; if $\E[a_r]=\E[b_r]=1$ for each $r$, then $c_r=1$ and $a_r=b_r$ marginally a.s.
\end{lemma}

\begin{proof}
Take logarithms; $u_r=\log a_r-\log b_r$ satisfies $\sum_r u_r(z_r)=0$ on $U$. Since $U$ is a product, varying $z_r$ alone with the rest fixed stays in $U$; subtracting two such equations shows $u_r$ is constant, $=\log c_r$, with $\prod_r c_r=1$. Normalization gives $c_r=\E[a_r]/\E[b_r]=1$.
\end{proof}

\begin{proposition}[Identifiability of one isolated edge]
\label{prop:edge-product}
Suppose $f(x_j)\prod_{k\in M}g_k(x_k)=f'(x_j)\prod_{k\in M}g'_k(x_k)$ on an open connected product set, with all functions satisfying Assumptions~\ref{ass:normalization} and \ref{ass:regularity}. Then $g_k=g'_k$ for every $k$ and $f=f'$ on $U_j$. Per-gate normalization alone suffices to remove the multiplicative scale ambiguity; no additional normalization of $f$ is required.
\end{proposition}

\begin{proof}
If $M=\varnothing$ the identity is $f=f'$ directly. Otherwise, gates are strictly positive, so we never take a logarithm of $f$. Choose $x_j^0$ with $f(x_j^0)\neq 0$ (possible by Assumption~\ref{ass:regularity}); evaluating at $x_j^0$ gives $\prod_k g_k=C\prod_k g'_k$ with $C=f'(x_j^0)/f(x_j^0)>0$. Apply Lemma~\ref{lem:factorization} to $\{g_k\}$ and $\{C^{1/|M|}g'_k\}$: this yields $g_k=c_k g'_k$ with $\prod_k c_k=C$, and normalization $\E[g_k]=\E[g'_k]=1$ forces $c_k=1$, hence $g_k=g'_k$ (pointwise by Lemma~\ref{lem:continuity-upgrade}) and $C=1$. With gates equal, the identity becomes $f(x_j)P=f'(x_j)P$ with $P=\prod_k g_k>0$; cancelling gives $f=f'$.
\end{proof}

Combining edge separability, support richness, regularity, and factorization identifiability yields the main population result.

\begin{theorem}[Population identifiability of normalized minimal \GNAVAR]
\label{thm:main}
Let $h_i$ admit two normalized minimal \GNAVAR{} representations satisfying Assumptions~\ref{ass:directsum}--\ref{ass:regularity}, observationally equivalent under $p$. Then:
\begin{enumerate}
\item the centered edge products agree, $F_{ij}=F'_{ij}$ $p$-a.s.\ for every $j$;
\item the modulator sets $M_{ij}$ are identified up to permutation and insertion/deletion of trivial gates;
\item for every active edge, after gate normalization the base function and all nontrivial gates are identified on the support ($f_{ij}=f'_{ij}$, $g_{ijk}=g'_{ijk}$); no residual scale gauge remains.
\end{enumerate}
\end{theorem}

\begin{proof}
\emph{Step 1.} Observational equivalence gives $\sum_j(F_{ij}-F'_{ij})=0$ $p$-a.s. By Assumption~\ref{ass:directsum}, $F_{ij},F'_{ij}\in\mathcal V_{ij}$, so $F_{ij}-F'_{ij}\in\mathcal V_{ij}$, and the direct-sum property forces each to vanish: $F_{ij}=F'_{ij}$. Assumption~\ref{ass:support} and Lemma~\ref{lem:continuity-upgrade} upgrade this to pointwise equality. This is claim (1).

\emph{Step 2.} By Definition~\ref{def:minimal}, $k\in M_{ij}$ iff $F_{ij}$ varies in $x_k$ on an open product subset, a property of $F_{ij}$ alone. Since $F_{ij}=F'_{ij}$, the essential-modulator sets coincide: $M_{ij}=M'_{ij}$ up to trivial gates. This is claim (2).

\emph{Step 3.} With $M_{ij}=M'_{ij}=:M$, both representations are products over the same modulator set with common edge product, so Proposition~\ref{prop:edge-product} identifies $g_{ijk}=g'_{ijk}$ and $f_{ij}=f'_{ij}$ on the support. Gate normalization alone removes the scale, so no residual gauge remains. This is claim (3).
\end{proof}

\noindent\emph{What this means in practice.} The theorem tells a practitioner the precise conditions under which a recovered modulator can be trusted as a property of the data rather than of the fit: the candidate modulators must enter through a separable edge structure (no two gates competing to explain the same variation) and their joint support must be rich enough that no gate-difference can hide off the observed region. When both hold, the recovered modulator identities, the gate shapes, and even the multiplicative scale are pinned down. The two conditions map directly onto the two instruments of Section~\ref{sec:diagnostic}: support richness is what the effective rank measures before fitting, and separability failures surface as the seed-to-seed disagreement the stability check detects after fitting.

\begin{corollary}[Impossibility under low-dimensional support]
\label{cor:lowdim}
If the joint support of a candidate modulator set lies in a lower-dimensional manifold, the conclusion can fail: distinct normalized nonconstant gates may agree on the observed support while differing off it.
\end{corollary}

\begin{proof}
Let $M=\{k,\ell\}$ with $x_\ell=\alpha x_k$ on the support. The normalized pair $g_k\propto e^{ax_k}$, $g_\ell\propto e^{bx_\ell}$ yields a product proportional to $e^{(a+\alpha b)x_k}$ on the support. Any $(a',b')$ with $a'+\alpha b'=a+\alpha b$ gives the same product there and satisfies the normalization, yet differs off the support. The one-parameter family is not reduced to a point by $\E[g]=1$, so support richness (Assumption~\ref{ass:support}) is necessary.
\end{proof}

\noindent\emph{What this means in practice.} Collinear candidate modulators are not merely a statistical-power problem that more data would solve; they make the recovery target genuinely undefined. The corollary is the formal reason the diagnostic is a \emph{pre-fit} test: because the obstruction lives in the support geometry, it can be detected from the inputs alone, before a single gate is trained, because no amount of optimization can identify a decomposition that the support itself does not distinguish.

\section{A Pre-Fit Diagnostic and a Stability Check}
\label{sec:diagnostic}

Theorem~\ref{thm:main} and Corollary~\ref{cor:lowdim} together say that identifiability hinges on the richness of the joint support of the candidate modulators: when that support collapses onto a lower-dimensional set, distinct gates agree on it and recovery is not identifiable. We now turn this support condition into a quantity computable \emph{before} any fit, and argue for the specific functional we use.

\paragraph{From support geometry to a second-moment proxy.}
The support condition is about the geometry of the region the lag-blocks occupy. Its linear shadow is exact and already decisive: if the joint lag-block covariance $\Sigma_S$ for a candidate set $S$ is rank-deficient, the lag-blocks lie in a proper linear subspace, and any two gate configurations differing only along an unoccupied direction are observationally indistinguishable, the degenerate case of Corollary~\ref{cor:lowdim}. So full-dimensional support is \emph{necessary} for identifiability, and the rank of $\Sigma_S$ measures it. Exact rank-deficiency is non-generic, however; in practice support is \emph{nearly} degenerate, concentrated near a lower-dimensional set, and identifiability degrades continuously as this happens (the $\rho$-sweep in Section~\ref{sec:exp-synth} exhibits exactly this continuous decay). We therefore need a continuous measure of how close $\Sigma_S$ is to rank-deficiency, not a binary rank.

\paragraph{Why effective rank.}
We use the participation ratio of the spectrum of $\Sigma_S$,
\[
 \reff(S) = \frac{(\operatorname{tr}\Sigma_S)^2}{\operatorname{tr}(\Sigma_S^2)} = \frac{\big(\sum_i \lambda_i\big)^2}{\sum_i \lambda_i^2},
\]
where $\lambda_i$ are the eigenvalues of $\Sigma_S$. It equals the ambient dimension $|S|\cdot K$ when the spectrum is flat (support spread equally across all directions), decreases continuously as eigenvalue mass concentrates, approaches $1$ as the support collapses onto one direction, is bounded in $[1,|S|\cdot K]$, and is invariant to the overall scale of $\Sigma_S$. It is thus a smooth count of how many directions the support effectively occupies, precisely the geometric content the theorem requires, and it is a closed-form function of the first two spectral moments, computable from one eigendecomposition with no training, no threshold, and no tuning. Alternative summaries are less aligned with the theorem's geometry. The condition number $\lambda_{\max}/\lambda_{\min}$ depends only on the extreme eigenvalues and ignores the bulk of the spectrum; numerical rank requires an arbitrary threshold and changes discontinuously; nonlinear measures such as mutual information or intrinsic-dimension estimators require density estimation and tuning.

\paragraph{A necessary-but-not-sufficient diagnostic} 
Because $\reff$ is a second-moment summary, its implication runs in only one direction: low $\reff$ \emph{certifies} approximate degeneracy and hence near-non-identifiability, whereas high $\reff$ is \emph{necessary but not sufficient}; support can be full-dimensional in covariance yet still harbor nonlinear degeneracies or interactions of higher order than the fitted form, which a linear summary cannot see. The diagnostic also presupposes a candidate set $S$: it scores whether a chosen set has rich enough support, not which variables to include. We therefore pair it with a post-fit \emph{two-seed stability check}: two independent fits should agree on the recovered modulators if the solution is reproducible, so disagreement is evidence against reliable recovery regardless of $\reff$. The instruments divide labor: $\reff$ is the cheap pre-fit screen that explains \emph{why} recovery should fail (support geometry), while the stability check is the post-fit test that detects \emph{that} it failed, including for reasons beyond the covariance's reach.

\subsection{A practical workflow for reliable interaction discovery}
\label{sec:workflow}

The theory and the two instruments combine into a simple procedure a practitioner can follow before trusting any discovered modulator (Figure~\ref{fig:workflow}).

\begin{figure}[t]
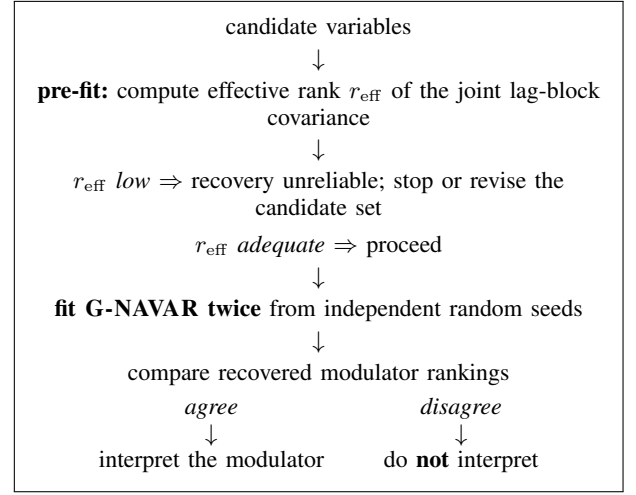

\centering
\setlength{\fboxsep}{6pt}
\fbox{\parbox{0.86\columnwidth}{\centering\small
candidate variables\\[2pt]
$\downarrow$\\[2pt]
\textbf{pre-fit:} compute effective rank $\reff$ of the joint lag-block covariance\\[2pt]
$\downarrow$\\[2pt]
\emph{$\reff$ low} $\Rightarrow$ recovery unreliable; stop or revise the candidate set\\[4pt]
\emph{$\reff$ adequate} $\Rightarrow$ proceed\\[2pt]
$\downarrow$\\[2pt]
\textbf{fit \GNAVAR{} twice} from independent random seeds\\[2pt]
$\downarrow$\\[2pt]
compare recovered modulator rankings\\[4pt]
\begin{tabular}{c@{\hspace{2.5em}}c}
\emph{agree} & \emph{disagree} \\
$\downarrow$ & $\downarrow$ \\
interpret the modulator & do \textbf{not} interpret \\
\end{tabular}
}}
\caption{Practical workflow. The effective-rank screen is a cheap pre-fit filter; the two-seed check is a post-fit falsification test. The asymmetry is deliberate (see text): instability is evidence against reliable recovery, but stability is supporting evidence, not proof of correctness.}
\label{fig:workflow}
\end{figure}

The asymmetry between the two outcomes is essential and easy to get wrong: instability is strong evidence \emph{against} reliable recovery, but stability is only supporting evidence, not proof of correctness. Two seeds can agree on a wrong answer if, for instance, the true interaction is of higher order than the fitted pairwise gates. The workflow therefore licenses a negative conclusion firmly and a positive one provisionally. When $\reff$ is low, the remedy is to revise the candidate set: drop mutually redundant variables, add genuinely independent ones - not to train longer.

\subsection{Failure modes}
\label{sec:failures}

The three empirical states do not arise for a single reason. The theory identifies four distinct mechanisms under which interaction discovery should not be trusted. \textbf{(F1) Concentrated support:} the candidate modulators are mutually collinear, so $\reff$ is low and distinct gates are observationally indistinguishable (Corollary~\ref{cor:lowdim}); this is the realized-volatility regime. \textbf{(F2) Missing modulator:} the true modulator is absent from the candidate set, so the diagnostic, which scores only the chosen set, cannot flag it; the two-seed check offers partial protection, since a model fit on an incomplete set often recovers unstably. \textbf{(F3) Finite-sample instability:} the population support is adequate but the sample is too small to realize it, so recovery is noisy despite population identifiability; this is the small-$T$ end of the synthetic recovery curve, and it is distinguishable from F1 because $\reff$ is adequate. \textbf{(F4) Order misspecification:} the true interaction is genuinely higher-order than the fitted base-times-gates form, so even with rich support and large samples the single-modulator reading is unstable while the model fits well - the regime discussed in Section~\ref{sec:order}. F1 and F2 are caught before fitting and after fitting respectively; F3 is mitigated by more data; F4 is the one that most resembles success while not being it, and is the strongest reason to treat a stable two-seed result as provisional rather than conclusive.

\section{Synthetic Experiments}
\label{sec:exp-synth}

\paragraph{Setup} We use a five-variable stationary system with target $x_1$ and four independent AR(1) sources, lag order $K=2$. The target is
\begin{align*}
 x_{1,t} = {}& f_{12}(x_{2,t-1})\,g_{123}(x_{3,t-1:t-2}) \\
 &+ f_{14}(x_{4,t-1})\,g_{145}(x_{5,t-1:t-2}) + 0.3\,x_{2,t-2} + \sigma_\epsilon\eta_t,
\end{align*}
with linear bases, a saturation-times-change gate $g_{123}$, a 2D Gaussian inhibition gate $g_{145}$, and $\sigma_\epsilon=0.1$ (noise floor $0.01$). The true modulators are $x_3$ for the $x_2$-edge and $x_5$ for the $x_4$-edge. \GNAVAR{} uses one base per source and one gate per (source, modulator) pair ($\approx 2{,}000$ parameters), trained with Adam and an $L^1$ penalty on gate deviation from one ($\lambda=0.005$) that operationalizes the minimality of Definition~\ref{def:minimal}. We score modulator-set recovery (both edges correct) and gauge-normalized gate $L^2$ error relative to the ground-truth gates.

\paragraph{Recovery at scale} Across $T\in\{1{,}000;5{,}000;25{,}000;100{,}000\}$ with five seeds each, modulator-set recovery converges monotonically: $0/5$ at $T=1{,}000$, $4/5$ at $5{,}000$, and $5/5$ at $T\ge 25{,}000$. Gate $L^2$ error declines with $T$ (e.g.\ $g_{123}$: $0.31\to0.27$; $g_{145}$: $0.29\to0.15$) and plateaus at moderate values, reflecting the finite-support smoothing discussed in Section~\ref{sec:theory}; the recovered gate shapes match the true 2D structure.

\paragraph{Additive baseline} Against a same-capacity additive (Pairwise) \NAVAR{} with no gates (517 vs.\ 2{,}065 parameters) at $T=100{,}000$ over five seeds, \GNAVAR{} achieves held-out MSE $0.0102\pm0.00006$ versus $0.1716\pm0.0012$, a $16.8\times$ advantage, winning every seed (Fig.~\ref{fig:baseline}). Both models' train and test MSE nearly coincide (Pairwise $0.170/0.172$), so Pairwise's plateau is a structural ceiling, not overfitting: an additive class can absorb a gate's mean but not its variance. The gap is largest on the high-variance gate $g_{123}$ (edge-contribution error $0.44$ vs.\ $0.26$) and narrows on the milder $g_{145}$ ($0.17$ vs.\ $0.15$), the expected signature of multiplicative structure.

\begin{figure}[t]
\centering
\includegraphics[width=\columnwidth]{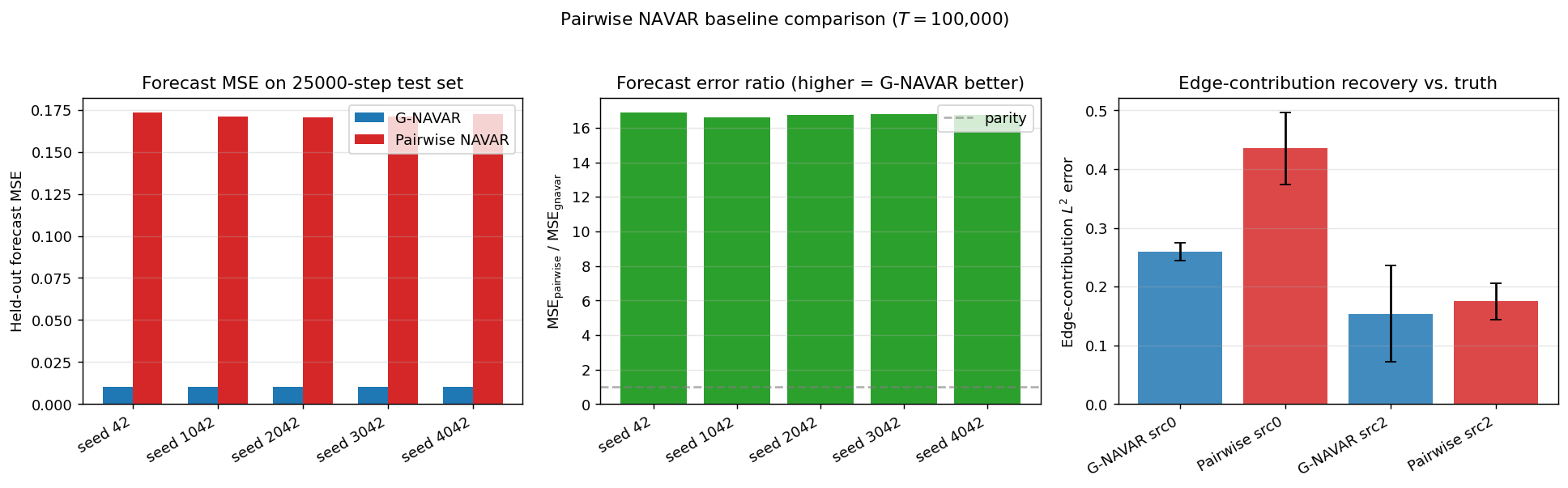}
\caption{Additive baseline at $T=100{,}000$. Left: per-seed held-out forecast MSE, \GNAVAR{} (blue) versus Pairwise \NAVAR{} (red). Center: the per-seed MSE ratio, tightly clustered near $16.8\times$ (parity dashed). Right: edge-contribution $L^2$ error to the true modulated edges; the gap is large on the high-variance gate (src0) and narrow on the milder one (src2), the signature of multiplicative structure additive models cannot absorb.}
\label{fig:baseline}
\end{figure}

\paragraph{Does gating earn its place?} The additive comparison shows interactions are necessary, but not that the multiplicative \emph{gating} form is. We therefore add two capacity-matched competitors (both $\ge$ \GNAVAR{} in parameters, so \GNAVAR{} cannot win by starvation): a black-box MLP over the flattened lag window (a pure forecaster, no interpretable structure), which tests whether \GNAVAR{} pays an accuracy price for interpretability; and an additive-plus-pairwise model $y=\text{bias}+\sum_j f_j(x_j)+\sum_{(j,k)} h_{j,k}(x_j,x_k)$ (a functional-ANOVA order-2 / GA\textsuperscript{2}M model) that \emph{can} represent the interaction but not via the gated parameterization, which tests whether gating is needed for recovery. Over five seeds at each of $T\in\{5{,}000;25{,}000;100{,}000\}$ (Table~\ref{tab:gating}), three findings emerge. First, additive \NAVAR{} is far behind on MSE ($\sim$16$\times$), confirming interactions matter. Second, \GNAVAR, the GA\textsuperscript{2}M model, and the black-box MLP are statistically indistinguishable on held-out MSE ($\le 0.0001$ apart at $T\ge 25{,}000$): \GNAVAR{} pays \emph{no} accuracy penalty for its interpretable gated form. Third, on \emph{recovery} of the true interacting pairs (measured by ANOVA-centered interaction mass for the GA\textsuperscript{2}M model and top-gate identity for \GNAVAR), both interaction-aware models recover the true interaction structure under this rich-support regime - the GA\textsuperscript{2}M model in $15/15$ runs, \GNAVAR{} in $12/15$ (its three misses occur under the same initialization seed across all sample sizes, a concrete instance of the seed-sensitivity the stability check is designed to detect). The reading is deliberately modest: under rich support, recoverability is governed primarily by the support condition rather than the particular interaction parameterization, and multiple interaction-aware models succeed. \GNAVAR{}'s contribution is not that it forecasts better, but that it is the parameterization for which identifiability is \emph{provable} (Section~\ref{sec:theory}); the empirical role of this comparison is to establish that this analyzability costs nothing in predictive power.

\begin{table}[t]
\centering
\small
\caption{Capacity-matched comparison on synthetic data (mean over five seeds)$^{\mathrm{a}}$} 
\label{tab:gating}
\setlength{\tabcolsep}{4pt}
\begin{tabular}{l c c c}
\toprule
Model & MSE ($T{=}25$k) & MSE ($T{=}100$k) & Recovers \\
\midrule
Additive \NAVAR        & $0.173$  & $0.172$  & --- \\
Black-box MLP          & $0.0114$ & $0.0106$ & n/a \\
GA\textsuperscript{2}M & $0.0115$ & $0.0106$ & $15/15$ \\
\GNAVAR                & $0.0113$ & $0.0106$ & $12/15$ \\
\bottomrule
\multicolumn{4}{p{0.95\columnwidth}}{\footnotesize
$^{a}$ Additive \NAVAR{} cannot represent interactions; the GA\textsuperscript{2}M model and black-box MLP can. \GNAVAR{} matches both on held-out MSE while remaining the only model in this comparison for which an identifiability guarantee is available. ``Recovers'' is the fraction of seeds recovering the true interacting structure.}
\end{tabular}
\end{table}

\paragraph{Support-collapse transition} Sweeping a correlation $\rho$ between the two modulators at $T=25{,}000$ (ten seeds, restart-and-keep-best), the effective rank $\reff(\{x_3,x_5\})$ decreases monotonically from $3.00$ to $1.54$ as $\rho\to1$, and modulator-set recovery, noisy across the mid-range, collapses to $0/10$ at $\rho\in\{0.99,1.0\}$ where $\reff\to K$. Plotting gate $L^2$ error directly against $\reff$ (Fig.~\ref{fig:effrank}) reveals the relationship more clearly than the correlation parameter$\rho$: error rises and its spread widens as $\reff$ falls toward $K$, the degenerate-support regime of Corollary~\ref{cor:lowdim}. Two independent fits on the same data also diverge there - an identifiability gap visible as the seed-dependence the theory predicts, supporting the claim that $\reff$, rather than $\rho$ itself, tracks the boundary of recoverability.

\begin{figure}[t]
\centering
\includegraphics[width=\columnwidth]{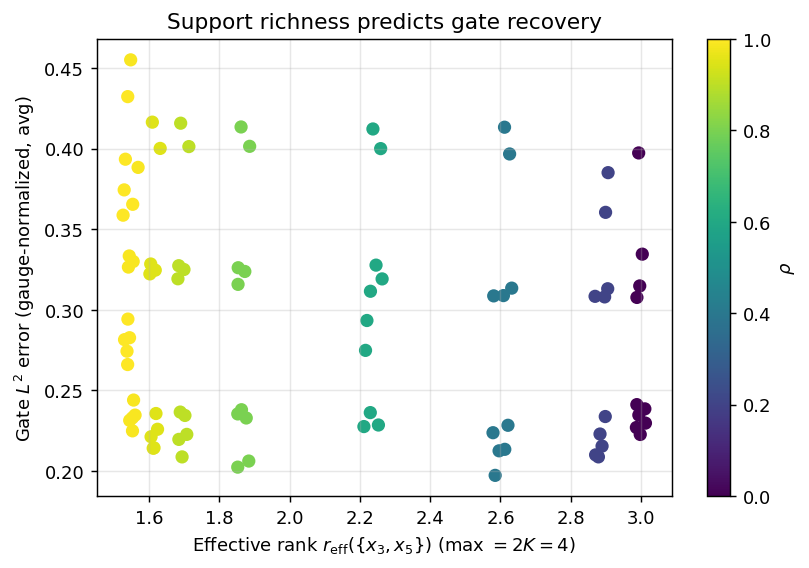}
\caption{Effective rank as a predictor of recovery. Gauge-normalized gate $L^2$ error versus the pre-fit effective rank $\reff(\{x_3,x_5\})$ (maximum $2K=4$), colored by the modulator correlation $\rho$. As $\reff$ falls toward $K=2$ (high $\rho$, right-to-left), error rises and spreads, marking the degenerate-support regime where gates are not identifiable.}
\label{fig:effrank}
\end{figure}

\section{Real-Data Experiments}
\label{sec:exp-real}

The synthetic study validates recovery where the structure is known; the decisive test is whether the diagnostic \emph{distinguishes} regimes on real data. We apply \GNAVAR{} to three domains that together exhibit all three empirical states the theory allows: recoverable structure under rich support (Beijing air quality), no stably recoverable structure \emph{despite} rich support (WDI development indicators), and no recovery because the support has collapsed (realized volatility). Table~\ref{tab:taxonomy} summarizes the pattern. The middle domain is the decisive one: without it, the taxonomy could suggest that rich support is \emph{sufficient} for recovery, whereas it shows that the stability check is required to separate the two high-rank regimes.

\begin{table}[t]
\centering
\caption{An empirical taxonomy of identifiability$^{\mathrm{a}}$} 
\label{tab:taxonomy}
\setlength{\tabcolsep}{3.5pt}
\begin{tabular}{lcccc}
\toprule
Domain & $\reff$ & Seed agr. & Interaction & Outcome \\
\midrule
Beijing & high & yes & strong & recoverable \\
WDI     & high & no  & weak   & unstable \\
RV      & low  & no  & ---    & support collapse \\
\bottomrule
\multicolumn{5}{p{0.95\columnwidth}}{\footnotesize
$^{a}$ The three domains realize the three empirical states predicted by the theory. Low rank precludes recovery; high rank separates into recoverable and unstable regimes according to the stability check.}
\end{tabular}
\end{table}

\subsection{Beijing air quality (rich support)}

The Beijing Multi-Site Air-Quality dataset~\cite{zhang2017beijing} provides hourly readings of six pollutants and six meteorological variables at twelve monitoring stations over 2013--2017; we use four geographically diverse sites (Nongzhanguan, Tiantan, Huairou, Gucheng), building lag tensors within contiguous clean runs and splitting chronologically $80/20$. We examine two textbook hypotheses: that temperature (TEMP) modulates the NO$_2\to$O$_3$ edge (ozone photochemistry, Setup A), and that wind speed (WSPM) modulates the PM2.5$\to$PM10 edge (dispersion, Setup B). We rank candidate modulators of the expected edge by gate triviality score $\E[(g-1)^2]$.

In Setup A, $\reff>4$ at all four sites, and \textbf{TEMP is the top-ranked modulator of NO$_2\to$O$_3$ at every site}, with margins over the second candidate ranging from $2.0\times$ to $59.3\times$ (Table~\ref{tab:beijing}). The recovered modulation pattern is consistent with textbook ozone photochemistry and emerges without supervision from observational data. In Setup B, $\reff<3.02$ at all sites; the expected modulator WSPM never ranks first (rank 3 at three sites, rank 2 at one), and dewpoint emerges as the top modulator at three of four sites, consistent with hygroscopic particle growth, which we report as the model's empirical finding rather than a confirmed mechanism. Across both setups, \GNAVAR{} wins held-out MSE in 6 of 8 cases (advantages 4--19\% where it wins). The predictive gains are modest relative to the synthetic experiments, which is expected because the real-data objective is interaction discovery rather than recovery of a known generating mechanism. The diagnostic separates the two regimes: $\reff>4$ coincides with rank-1 recovery in every Setup A case; $\reff<3$ coincides with non-recovery in every Setup B case (Fig.~\ref{fig:beijing-fig}).

\begin{figure*}[t]
\centering
\includegraphics[width=0.92\textwidth]{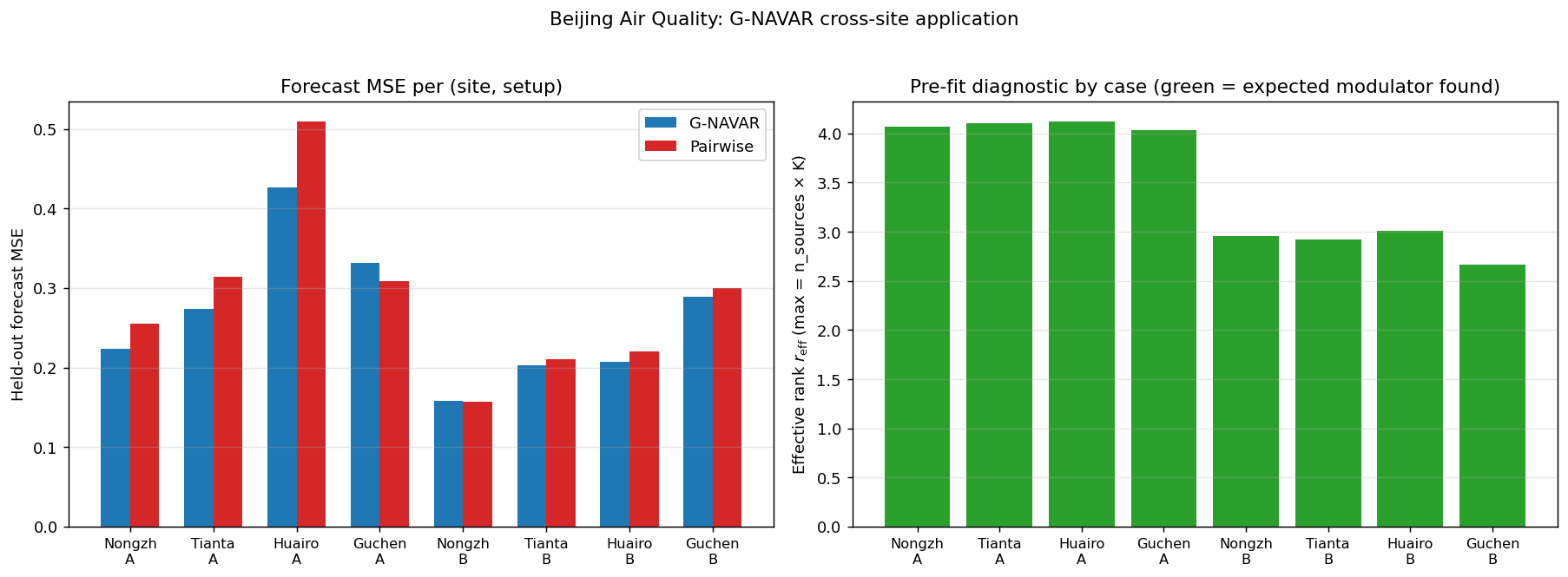}
\caption{Beijing air quality, four sites $\times$ two setups (A: NO$_2\to$O$_3$; B: PM2.5$\to$PM10). Left: held-out forecast MSE, \GNAVAR{} (blue) versus Pairwise (red); \GNAVAR{} wins 6 of 8 cases. Right: pre-fit effective rank $\reff$ per case. The Setup A cases sit at $\reff>4$ (where TEMP is recovered as the top modulator at all four sites); the Setup B cases at $\reff<3$ (where the expected modulator is not recovered). The pre-fit $\reff$ separates the two regimes before any fitting.}
\label{fig:beijing-fig}
\end{figure*}

\begin{table}[t]
\centering
\caption{Beijing Setup A (NO$_2\to$O$_3$)$^{\mathrm{a}}$} 
\label{tab:beijing}
\begin{tabular}{lccc}
\toprule
Site & $\reff$ & Top modulator & Margin over 2nd \\
\midrule
Nongzhanguan & 4.065 & TEMP & $10.9\times$ \\
Tiantan & 4.107 & TEMP & $2.0\times$ \\
Huairou & 4.119 & TEMP & $16.9\times$ \\
Gucheng & 4.035 & TEMP & $59.3\times$ \\
\bottomrule
\multicolumn{4}{p{0.95\columnwidth}}{\footnotesize
$^{a}$ TEMP is the top-ranked modulator at all four sites, with margins ranging from $2.0\times$ to $59.3\times$ over the second-ranked candidate.}
\end{tabular}
\end{table}

\subsection{WDI development indicators (rich support, weak structure)}
\label{sec:wdi}

The Beijing result alone might suggest that rich support is sufficient for recovery. It is not, and our second real domain shows why. We use a World Development Indicators panel~\cite{worldbank2024wdi} (265 units, 1970--2023) and ask, in the spirit of the conditional resource-curse hypothesis, whether the effect of resource rents on GDP growth is modulated by the investment environment. We pre-committed to a prediction: among candidate modulators (investment, trade openness, FDI), investment ranks first, seed-stably. The candidate set screens as rich, $\reff = 4.47$ of $4$, so the diagnostic indicates that recovery is feasible. We winsorize all series at the training $1$st/$99$th percentiles, standard treatment for a heavy-tailed growth target.

The prediction was not borne out, and the manner of its failure is the point. This is precisely the high-rank regime for which the theory predicts that support richness alone is not decisive. The two seeds disagree on the top modulator (investment versus trade openness); the gate scores are small and their slopes change sign across seeds; and \GNAVAR{} does not improve on the additive baseline (held-out MSE $1.26$ versus $1.19$). We state the conclusion at the strength the evidence supports: there is \emph{no evidence of a strong, stably recoverable interaction under this specification}. We do not claim that no interaction exists. Rather, the recovered modulators are not reproducible under independent optimization and therefore cannot be interpreted reliably. This is the regime the effective rank cannot rule out by construction (Section~\ref{sec:diagnostic}): the support is adequate, so the obstruction is not degeneracy, and only the post-fit stability check reveals that the apparent modulator is not reproducible.

\subsection{Realized volatility (collapsed support)}
\label{sec:rv}

Our third domain is the negative-control case: one that the diagnostic flags as collapsed before any fit. We use daily realized volatility for eight international equity indices~\cite{son2023forecasting}: S\&P 500, DAX, CAC 40, FTSE 100, OMX Stockholm, Nikkei 225, KOSPI, and Hang Seng, spanning the US, Europe, and Asia, with $2{,}615$ trading days per series, constructed from intraday returns and log-transformed and z-scored within series. Unlike Beijing's meteorological drivers or WDI's heterogeneous development indicators, realized-volatility series are highly persistent (lag-1 autocorrelation $0.76$--$0.84$) and strongly cross-correlated (up to $0.94$): the markets move together, so the joint lag-block support concentrates in only a few directions. The diagnostic captures this directly: for all four targets we examine (FTSE, GDAXI, N225, HSI), the pre-fit $\reff<2$. The literature treats the S\&P 500 as the global volatility leader, so a natural hypothesis is that SPX dominates as a modulator of within-region spillover - the kind of hypothesis that succeeded on Beijing. We pre-committed to the prediction that recovery would be \emph{unreliable}, and tested four consequences.

Three of four held. SPX ranks first in only $5/12$ peer edges; the top modulator differs across all four targets (SPX, FCHI, HSI, FTSE); and two independent seeds agree on the top modulator only $44\%$ of the time (Table~\ref{tab:rv}). \GNAVAR{} beats the additive baseline in just $1/4$ cases (versus $6/8$ on Beijing). The one failed prediction is instructive. The failure occurs in exactly the dimension the theory does not require: within-fit margins are \emph{not} small ($1.9$--$5.8\times$). Each fit confidently selects a dominant modulator, but \emph{which} one is unstable across seeds, the model is confidently arbitrary, not appropriately uncertain. This is arguably a sharper signature of non-identifiability than small margins would be, and it is exactly what the seed-stability check detects.

\begin{table}[t]
\centering
\caption{Realized volatility (collapsed support)$^{\mathrm{a}}$}
\label{tab:rv}
\setlength{\tabcolsep}{4pt}
\begin{tabular}{lccccc}
\toprule
Target & $\reff$ & Top mod. & Margin & Seed agr. & MSE ratio \\
\midrule
FTSE  & $1.50$ & SPX  & $5.5\times$ & $0.50$ & $1.00\times$ \\
GDAXI & $1.49$ & FCHI & $1.9\times$ & $0.25$ & $0.99\times$ \\
N225  & $1.77$ & HSI  & $5.8\times$ & $0.25$ & $0.95\times$ \\
HSI   & $1.97$ & FTSE & $4.8\times$ & $0.75$ & $0.93\times$ \\
\bottomrule
\multicolumn{6}{p{0.95\columnwidth}}{\footnotesize
$^{a}$  Pre-fit $\reff<2$ at every target; the top modulator differs across all four, within-fit margins are deceptively large, and two seeds agree on the top modulator under half the time, consistent with the support-collapse regime predicted by the theory. MSE ratio is additive/\GNAVAR{} ($\le 1$ means no \GNAVAR{} advantage).}
\end{tabular}
\end{table}

\subsection{The contrast}

The three domains (Table~\ref{tab:taxonomy}) realize the three regimes under a single model and pipeline: clean cross-site recovery under rich support (Beijing), no stably recoverable structure despite rich support (WDI), and unstable selection under collapsed support (realized volatility). The pre-fit screen and the post-fit stability check together identify which regime a practitioner occupies before any discovered interaction is interpreted.

\section{Related Work}
\label{sec:related}

\GNAVAR{} extends additive neural Granger methods~\cite{bussmann2021navar,tank2022neural} to interaction-aware dynamics. Our decomposition and its identification draw on additive and functional ANOVA models~\cite{stone1985additive,hastie1990generalized,gu2002smoothing} and, crucially for dependent inputs, the generalized Hoeffding--Sobol decomposition~\cite{hooker2007generalized,chastaing2012generalized}. The identifiability question parallels uniqueness results in tensor decompositions~\cite{kruskal1977threeway,kolda2009tensor,anandkumar2014tensor} and identifiability in nonlinear ICA~\cite{hyvarinen2019nonlinear,khemakhem2020variational}; our impossibility results echo the support-dependent disentanglement impossibility of~\cite{locatello2019challenging}. Unlike these literatures, however, our primary object is not representation learning itself but the recoverability of interaction structure from observational time series.

A large body of work \emph{discovers} interactions in time series: neural relational inference infers a latent interaction graph among components~\cite{kipf2018neural}; attention-based forecasters expose pairwise dependencies through learned attention weights~\cite{vaswani2017attention,wu2020connecting}; and statistical interaction detection reads interactions off trained network weights~\cite{tsang2018detecting} or selects them under hierarchy constraints~\cite{friedman2008predictive,bien2013lasso}. These methods address a problem complementary to ours: they propose \emph{which} interactions are present, not \emph{whether} a proposed interaction is recoverable from the observational distribution. Because the obstruction we study is a property of the input support rather than the fitted model, it applies to these methods as much as to \GNAVAR: an attention map or inferred edge computed on collapsed support faces the same instability, with recovered structure reflecting optimization artifacts rather than uniquely identifiable interactions. Our contribution is to make that obstruction precise in a setting where it can be analyzed formally, to characterize when interaction recovery is identifiable, and to supply a pre-fit diagnostic that can be applied before trusting a discovered interaction.

\section{Discussion: Identifiability of Interaction Order}
\label{sec:order}

The WDI result suggests a broader question than gate recovery alone. Our theory takes the interaction structure as given, a base function multiplied by single-variable gates, and asks when those gates are identifiable. The deeper question is whether the \emph{order} of an interaction, that is, whether a target dependence is two-way, three-way, or higher, is itself identifiable from the observational distribution. The question matters because a model that assumes the wrong order is misspecified in a way that no amount of data fixes. The functional ANOVA decomposition gives the order a representation-independent definition of interaction: the order of a dependence is the highest-order orthogonal component with nonzero variance that cannot be absorbed into lower-order terms. Whether that component is identifiable is then governed by the same support conditions that govern gate identifiability here. On full-dimensional product support a genuinely $k$-way function cannot be written as a sum of lower-order terms, so the order is identified; under rank-collapsed support it can: a three-way product $x_1x_2x_3$ restricted to $x_3=x_1$ is exactly the two-way function $x_1^2x_2$, and the order is unidentifiable. Between these extremes, strong correlation can degrade order identifiability well before complete rank collapse: as inputs become increasingly collinear, higher-order interactions become progressively harder to distinguish from lower-order structure. This is the order-level analog of the gate non-identifiability our diagnostic detects: when candidate modulators are mutually correlated, the data may not determine the interaction order, and a model forced to commit to one will commit arbitrarily, producing rich effective rank yet unstable recovery under reseeding. A full treatment of when interaction order is identifiable, and of estimators that recover it, is beyond our scope here; we note only that it appears to be governed by conditions of the same character as those in our main theorem.

\section{Conclusion}
\label{sec:conclusion}

Interaction capacity is not interaction identifiability. We showed that gated neural autoregression can represent effect modification but that its gates are not recoverable from observational data without structural conditions: dependence lets edge terms leak, and low-dimensional support makes distinct gates indistinguishable. Under a direct-sum separability condition, rigorous for disjoint supports, and extended to shared modulators via a hierarchically orthogonal decomposition under stated assumptions, normalized minimal \GNAVAR{} is identified up to permutation and trivial gates, with scale fixed by gate normalization. The resulting picture is a necessary-but-not-sufficient one: rich support makes recovery feasible, but only stability determines whether a recovered interaction is reproducible. The theory yields a pre-fit effective-rank diagnostic and a post-fit stability check that, across three real domains, distinguish three empirical regimes: recoverable interaction structure, unstable recovery despite adequate support, and non-identifiability under support collapse. More broadly, interaction discovery is constrained not only by model capacity but by the geometry of the observational distribution itself. 

All experiments were run on NVIDIA GPU via Colab/Blackwell. Total runtime end-to-end is about 22 minutes. All code and result artifacts are available at an anonymized repository.\footnote{\texttt{https://anonymous.4open.science/r/ICDM-GNAVAR-EDAE/}}

\end{document}